%% file: main.tex
\documentclass[runningheads,envcountsect]{llncs}

\usepackage{graphicx}
\usepackage{amsmath, amssymb}
\usepackage{bbm}
\usepackage{babel}
\usepackage{hyperref}
\usepackage{tikz}
\usepackage{microtype}
\usepackage{orcidlink}
\usepackage{xcolor}
\usepackage{tikz}
\usetikzlibrary{calc}
\usetikzlibrary{arrows.meta}
\usetikzlibrary{decorations.pathreplacing}
\usepackage{float}  
\usepackage{subcaption}
\usepackage{wrapfig}
\usepackage{tikz}
\usetikzlibrary{positioning, automata}
\renewcommand{\vec}[1]{\boldsymbol{#1}}

\usepackage[capitalize,nameinlink]{cleveref}

\crefname{theorem}{Theorem}{Theorems}
\crefname{lemma}{Lemma}{Lemmas}
\crefname{proposition}{Proposition}{Propositions}
\crefname{corollary}{Corollary}{Corollaries}
\crefname{definition}{Definition}{Definitions}
\crefname{example}{Example}{Examples}
\crefname{assumption}{Assumption}{Assumptions}
\crefname{condition}{Condition}{Conditions}

\ifodd 0
    \newcommand{\rev}[1]{{\color{blue}#1}}
    \newcommand{\com}[1]{\textbf{\color{red}(COMMENT: #1)}}
\else
    \newcommand{\rev}[1]{#1}
    \newcommand{\com}[1]{}
\fi

\title{When In Doubt, Abstain: The Impact of Abstention on Strategic Classification}
\titlerunning{The Impact of Abstention on Strategic Classification}

\author{
  Lina Alkarmi\,\orcidlink{0000-0003-3097-4781} \and
  Ziyuan Huang\,\orcidlink{0000-0002-8939-6456}  \and
  Mingyan Liu\,\orcidlink{0000-0003-3295-9200} 
}
\institute{University of Michigan, Ann Arbor, MI 48109, USA \\
\texttt{\{lalkarmi, ziyuanh, mingyan\}@umich.edu}}

\begin{document}
\maketitle
\begin{abstract}
Algorithmic decision making is increasingly prevalent, but often vulnerable to strategic manipulation by agents seeking a favorable outcome. Prior research has shown that classifier abstention (allowing a classifier to decline making a decision due to insufficient confidence) can significantly increase classifier accuracy. This paper studies abstention within a strategic classification context, exploring how its introduction impacts strategic agents' responses and how principals should optimally leverage it. We model this interaction as a Stackelberg game where a principal, acting as the classifier, first announces its decision policy, and then strategic agents, acting as followers, manipulate their features to receive a desired outcome. Here, we focus on binary classifiers where agents manipulate observable features rather than their true features, and show that optimal abstention ensures that the principal's utility (or loss) is no worse than in a non-abstention setting, even in the presence of strategic agents. We also show that beyond improving accuracy, abstention can also serve as a deterrent to manipulation, making it costlier for agents, especially those less qualified, to manipulate to achieve a positive outcome when manipulation costs are significant enough to affect agent behavior. These results highlight abstention as a valuable tool for reducing the negative effects of strategic behavior in algorithmic decision making systems.

\end{abstract}

\input{introduction}

\input{model}

\input{uniform_example}

\input{simulation}

\section{Conclusion}
\label{sec:conclusion}
This paper studied abstention strategies in strategic classification, modeled as a Stackelberg game where the principal moves first and agents respond strategically. We showed that optimal abstention never increases the principal's loss and can deter manipulation when costs are sufficiently high. Principals also tend to abstain more in the presence of strategic agents. Using a linear classifier with uniformly distributed agents, we analyzed optimal thresholds and manipulation equilibria under varying costs. Our theoretical results were accompanied by simulations that illustrated the effects of manipulation cost, abstention cost, and noise on optimal abstention behavior. Future directions include a more rigorous analysis of explicit uncertainty in the classification process, such as the modeling of noisy features we simulated.

\bibliographystyle{splncs04}
\bibliography{references}

\input{appendix}
\end{document}

%% file: introduction.tex
\section{Introduction}
\label{intro}

With the proliferation of data and machine learning (ML) algorithms, algorithmic decision making is becoming more and more common, including in many areas of (cyber)security. These include training algorithms to detect security threats like malware and unauthorized access. Algorithm decisions are fast, enabling real-time response, and can be highly accurate.  At the same time, algorithms are also prone to manipulation by those who may not otherwise receive a favorable decision outcome. In the cybersecurity context, a typical example would be an adversary who designs malware with the goal of evading detection by a classifier trained to catch malicious software. Such an attempt often involves obfuscating or modifying key code signatures of malware, knowing that these features are what the algorithm has been trained to look for. Similarly, spammers are always adapting and rewriting their email content to evade the detection of mailbox filters. In a non-cybersecurity context, an example would be someone who cheats on an exam that they need to pass, or a job seeker who lies on their resume in order to pass algorithmic filtering.

The field of strategic classification has emerged to address this type of manipulation \cite{hardt_strategic_2015,perdomo_performative_2021,dong_strategic_2017,chen_learning_2020,zrnic_who_2022}. These problems are typically modeled as a Stackelberg game where a principal first designs a classifier and commits to its policy, and then strategic agents respond by manipulating their features to obtain a favorable outcome \cite{hardt_strategic_2015}. Prior work has focused on designing classifiers that incentivize honest behavior \cite{kleinberg_how_2019,jin_incentive_2022,bechavod_gaming_2021} or on using randomized rules to create an optimal stochastic policy \cite{10485755}.

In a parallel development, there has been an increasing interest in the idea of {\em abstention} as an additional option for a classifier to decline to provide a classification decision when it lacks sufficient confidence.  The ability to abstain from making a decision has been shown to significantly increase the accuracy of a classifier, even if it comes at the expense of leaving some tasks undecided (which in practice may then need to go through manual inspection, etc.) \cite{el-yaniv_foundations_2010,ortner_learning_2016,geifman_selective_2017}. 

In this paper we are interested in understanding how the introduction of abstention impacts the response of a strategic agent, and in turn, how the principal should optimally determine when to abstain in anticipation of the agent's best response. To the best of our knowledge, this is the first study that considers abstention in a strategic classification context. We will limit ourselves to  binary classifiers, as is the case with the vast majority of the literature on strategic classification. We will also limit our attention to the case where the agent cannot change its true features (or its true label); it can only change what is observed by the classifier. In other words, the agent can change the way it looks at a cost but not its substance. This means that under this model an agent is not allowed to make an honest effort to improve its true feature or label, a type of model studied in \cite{braverman_role_2020,hardt_strategic_2015,miller_strategic_2020}. To provide a foundational and analytically tractable analysis of this problem, we will use a one-dimensional case study with a simple threshold-based classifier. We leave the extension to multi-dimensional settings, such as with linear classifiers and Gaussian feature distributions, for future work.

We will derive the solution to a case of this new Stackelberg game and show the optimal decisions by both the decision maker and the agent.  Our main findings are as follows: 
\begin{enumerate}
\item The ability to abstain, provided it is done optimally, always allows the principal to attain a utility that is no worse than without this ability, regardless of the cost of abstention.  This is true in the absence of a strategic agent (as shown in the literature), and continues to hold in the presence of one (this paper). 

\item While abstention is typically used to improve the principal's accuracy (and thus utility), we show that in the presence of a strategic agent, when manipulation costs are sufficiently but not prohibitively high, abstention can also serve as a deterrent to strategic manipulation, by effectively making it harder/costlier for the agent to manipulate. 

\end{enumerate} 

The remainder of the paper is organized as follows. Section \ref{sec:problem} defines the problem and Stackelberg game formulation. In Section \ref{sec:fixed_f_analysis}, we characterize the optimal abstention function for a fixed classifier. Section \ref{sec: uniform_example} presents a case study using a linear classifier and uniform agent feature distribution, where we analyze the principal's optimal abstention, equilibrium without abstention, and expected manipulation of unqualified agents. Section \ref{sec:sim} provides simulation results, including the impact of system parameters, and an assessment of abstention's ability to reduce harm from strategic agents. Finally, Section \ref{sec:conclusion} concludes the paper.

%% file: model.tex
\section{Problem Formulation}
\label{sec:problem}
 We consider a strategic classification setting where an agent possesses a true feature vector $\vec{x}\in\mathcal{X}$ and a true label $y\in\{0,1\}$, jointly distributed according to the distribution $\mathcal{D}$. \rev{Let $\mathcal{X} \subseteq \mathbb{R}^d$ denote the feature space, representing the set of all possible true and observable feature vectors for an agent.} The principal interacts with the agent through a {\em Stackelberg game}: the principal first commits to a classifier $f:\mathcal{X}\to\{0,1\}$ that produces predictions and a rejection/abstention function $r:\mathcal{X}\to\{0,1\}$, after which the agent strategically manipulates its feature vector to $\hat{\vec{x}}\in\mathcal{X}$ to maximize its {\em utility}:
\begin{align}\label{eq:agent-utility}
    U({\vec{z}}|\vec{x}):=f({\vec{z}})r({\vec{z}})-\gamma\cdot\textrm{dist}({\vec{z}},\vec{x}).
\end{align}
Here, $\textrm{dist}(\cdot,\cdot)$ is some distance measure over $\mathcal{X}$ and $\gamma>0$ scales the distance into manipulation cost. 
The input or the observable feature to the classifier and the abstention function is the manipulated feature $\hat{\vec{x}}$. 
The agent only benefits from accepted positive decisions of the principal (the first term), determined by the two functions $f$ and $r$: 
the classifier generates a prediction (estimated label) $\hat{y}$ via $\hat{y}:=f(\hat{\vec{x}})$, which is then accepted (resp. abstained) by the principal if $r(\hat{\vec{x}})=1$ (resp. $\hat{\vec{x}}=0$).
The principal's objective is to minimize the {\em expected loss} in anticipation of the agent's manipulation:
\begin{equation}\label{eq:principal_loss}
    \begin{aligned}
\min_{f,r}\;&L(f,r):=\mathbb{E}_{\vec{x},y\sim\mathcal{D}}[l(f(\hat{\vec{x}}),y)r(\hat{\vec{x}})+c(1-r(\hat{\vec{x}}))] \\
\textrm{s.t.}&\;\hat{\vec{x}}\in\arg\max_{\vec{z}\in\mathcal{X}}\;U(\vec{z}|\vec{x})~
\end{aligned}
\end{equation}
where $l:\{0,1\}^2\to[0,1]$ is the {\em pointwise loss function} where $l(f(\vec{x}),y)$ represents the loss incurred for a single prediction $f(\vec{x})$ given the true label $y$. We assume $l(0,0)=l(1,1)=0$ without loss of generality. The principal faces an abstention cost $c \in [0,1]$ when refraining from making a decision. This reflects the extra resources that may be needed to process the rejections or opportunity cost due to reduced classification coverage. For example, if an intrusion detection system is uncertain about a potential threat, the firm might incur additional manual investigation costs to resolve the ambiguity. \rev{In the rest of the paper, we also refer to Eq. \eqref{eq:principal_loss} as the {\em constrained} problem and the problem without the strategic agent as the {\em unconstrained} problem. We will also use the terms {\em principal} and {\em decision maker} interchangeably.}

\rev{In words, the principal's goal is to minimize a combined cost of making a classification mistake (with a unit cost of 1) and of declining to make a classification decision measured by the pointwise loss $l$. On the other hand, the agent decides to maximize its reward from a positive decision (unit reward of 1) less a quadratic cost of manipulation. This work focuses on deriving the optimal abstention function $r$ for a fixed (and potentially suboptimal) classifier $f$.}

%% file: uniform_example.tex
\section{Optimal Abstention for a Fixed Classifier}
\label{sec:fixed_f_analysis}

In this section, we focus on characterizing the optimal abstention function given a fixed classifier, in the presence of a strategic agent.
In doing so, we will analyze the principal's expected loss and examine the difference between using a classifier with an abstention mechanism or strategic manipulation and one without.

\subsection{Agent's Best Response}
\label{ssec:agent_best_response_fixed_f}

According to the agent's utility in Eq. \eqref{eq:agent-utility}, its best response $\hat{\vec{x}}$ is a function of the game parameters, also written as $\hat{\vec{x}}(\vec{x}, \gamma, f,r)$. 
The best response is the result of balancing the desire to achieve a positive, accepted classification against the cost of manipulation.
Given the agent's true feature vector $\vec{x}$ and manipulation cost $\gamma$, its best response falls into two scenarios:
\begin{enumerate}
\item If there exists a manipulated feature $\hat{\vec{x}}$ such that $f(\hat{\vec{x}})=r(\hat{\vec{x}}) = 1$ and $\gamma\cdot\textrm{dist}(\hat{\vec{x}},\vec{x})\leq 1$, then there is incentive for the agent to manipulate. Its best response $\hat{\vec{x}}^{*}$ is the cheapest of these features to manipulate.

\item If no such $\hat{\vec{x}}$ exists (that achieves a positive classification and with manipulation cost $\leq 1$), then the agent has no incentive to manipulate and its best response is simply its true feature vector $\vec{x}$.
\end{enumerate}

The exact expression of the agent's best response depends on the classifier $f$ and the abstention function $r$. In Section \ref{sec: uniform_example} we will consider a specific example where the full expression of the agent's best response is obtained.

\subsection{Principal's Expected Loss and the Advantage of Abstention}
\label{ssec:principal_loss_fixed_f}

Given the agent's best response $\hat{\vec{x}}(\vec{x}, \gamma, f,r)$, the principal's expected loss $L(f,r)$ can be written as follows:
\begin{eqnarray}
L(f,r) = \mathbb{E}_{\vec{x},y\sim\mathcal{D}}\left[l(\hat{\vec{x}}(\vec{x}, \gamma, f,r),y)r(\hat{\vec{x}}(\vec{x}, \gamma, f,r))+c(1-r(\hat{\vec{x}}(\vec{x}, \gamma, f,r)))\right].
\end{eqnarray}

With a fixed classifier $f$, minimizing this function with respect to $r$ yields the optimal abstention function $\bar{r}^*$. The specific calculation of this expectation requires knowledge of the distribution $\mathcal{D}$ and the fixed form of $f$. Below is a general result that says it is always in the principal's interest to have the option to abstain.

\begin{theorem}\label{thm:benefit_abstention}
For any fixed classifier $f$, any data distribution $\mathcal{D}$, and any principal's cost of abstention $c$, the minimum expected loss achievable with an optimal abstention function $\bar{r}^*$ is less than or equal to the expected loss without abstention. That is:
$$ L(f,\bar{r}^*) \le L_{\textrm{no\_abstention}}(f) $$
where $L_{\textrm{no\_abstention}}(f)$ is the expected loss under $f$ with no abstention applied.
\end{theorem}

\begin{proof}

Let $\mathcal{R}$ be the set of all valid abstention functions $r:\mathcal{X}\to\{0,1\}$. The principal's problem is to choose an optimal abstention function $\bar{r}^* \in \mathcal{R}$ that minimizes their expected loss $L(f,r)$, given the agent's best response to $(f,r)$. Consider the case where the principal sets $r(\vec{\hat{x}}) = 1$ for all $\hat{\vec{x}}\in\mathcal{X}$, meaning the principal never abstains. This choice of $r(\hat{\vec{x}})$ effectively reduces the principal's model to that of no-abstention. If the principal adopts this $r(\hat{\vec{x}})$, the agent's utility function simplifies to $U(\vec{\hat{x}}) = f(\hat{\vec{x}}) - \gamma\cdot\textrm{dist}(\hat{\vec{x}},\vec{x})$, which is the same as the agent's utility function in the no-abstention setting. Thus, $\hat{\vec{x}}(\vec{x}, \gamma, f, r = 1) = \hat{\vec{x}}_{\textrm{no\_abstention}}(\vec{x},\gamma,f)$. Therefore, $L(f,r = 1) = L_{\textrm{no\_abstention}}(f)$. Since $\bar{r}^*$ is the optimal abstention function, it minimizes $L(f,r)$ over the entire set of valid abstention functions $\mathcal{R}$. Therefore, by definition of optimality, $L(f,\bar{r}^*) \le L(f,r = 1)$. Substituting this, we conclude that $L(f,r^*) \le L_{\textrm{no\_abstention}}(f)$. \qed

\end{proof}

\subsection{Comparison of Optimal Abstention: Strategic vs. Non-Strategic}

This section characterizes the optimal abstention function given a fixed classifier in a strategic (constrained) setting, in comparison to the optimal abstention function when the agent is non-strategic (unconstrained).
We refer to the solution to the principal's problem as the {\em constrained} (or strategic) solution, denoted as $\bar{r}^*$, and the solution where the agent does not manipulate the {\em unconstrained} (or non-strategic) solution, denoted as $r^*$. Define $L_f(\vec{x}):=\mathbb{E}_{y}[l(f(\vec{x}),y)\big\vert\vec{x}]$ as the classifier's {\em conditional loss} on $\vec{x}$. It's a standard result from the literature that the following function
\begin{equation}\label{eq:general-r-only}
r^*(\vec{x}) = \begin{cases}
1 & L_f(\vec{x})\leq c\\
0 & L_f(\vec{x})> c
\end{cases}
\end{equation}
is a solution to the unconstrained problem. The solution is unique up to the case $L_f(\vec{x})= c$ where the principal is indifferent between abstention and not. Intuitively, the principal chooses to abstain when the expected loss associated with the classification decision exceeds the fixed cost of abstention. The next result highlights that the core difference between the constrained and unconstrained optimal abstention functions lies essentially in the positively classified data points. 
\begin{theorem}\label{theorem:negative-coincides}
When $f$ is given and $\bar{r}^*$ is an optimal abstention function to the constrained problem, then the abstention function $\tilde{r}^*$ such that $\forall\hat{\vec{x}}\in\mathcal{X}$,
\begin{align}
&f(\hat{\vec{x}})=1\implies \tilde{r}^*(\hat{\vec{x}})=\bar{r}^*(\hat{\vec{x}})\label{thm:eq:positive} \\
&f(\hat{\vec{x}})=0\implies \tilde{r}^*(\vec{\vec{x}})=r^*(\hat{\vec{x}})\label{thm:eq:negative}
\end{align}
is also an optimal abstention function for the constrained problem.
\end{theorem}

What this result says is that we can always find an equally optimal constrained abstention function that coincides with $r^*$ on $\{f(\hat{\vec{x}})=0\}$. Intuitively, agents would never manipulate to $\hat{\vec{x}}$ if $f(\hat{\vec{x}})=0$, regardless of the choice of abstention. Consequently, the post-response density at $\hat{\vec{x}}$ is either zero (if the agent with true feature $\hat{\vec{x}}$ has an incentive to manipulate) or the same as the pre-response density at $\hat{\vec{x}}$ (otherwise). If the agent manipulates, the value of $\bar{r}^*(\hat{\vec{x}})$ can be arbitrarily chosen because its zero post-response density means it contributes nothing to the principal's expected loss; conversely, if the agent does not manipulate, $r^*(\hat{\vec{x}})$ is already the optimal abstention decision for $\hat{\vec{x}}$ under the pre-response distribution. Therefore, reusing $r^*$ on negatively classified data points for the constrained solution remains optimal.

In light of this, we assume in the following that $\bar{r}^*$ are chosen to satisfy Eq. \eqref{thm:eq:negative}.
For $f(\hat{\vec{x}})=0$, we have $L_f(\hat{\vec{x}})=l(0,1)p(y=1|\hat{\vec{x}})$. Thus, the principal tends to abstain from more negative data points when it becomes more false-negative averse (i.e., when $l(0,1)$ is higher). This holds for both constrained and unconstrained solutions.

\begin{theorem}\label{theorem:no-larger-accept}
If $\bar{r}^*$ is an optimal constrained abstention function and satisfies Theorem \ref{theorem:negative-coincides}, and $f$ is ``informative'' (i.e., $p(y=1|\vec{x})\geq p(y=1|\vec{z})\iff f(\vec{x})\geq f(\vec{z})$), then we must have $\bar{r}^*\not> r^*$.
\end{theorem}
In other words, the principal should not abstain less in optimality when anticipating potential manipulation from the agents.

\section{Case Study: One-Dimensional Uniform Distribution}
\label{sec: uniform_example}

We next examine closely a specific case in one dimension (scalar features) to gain further insights into the impact of abstention on the agent's best response. We assume $l(0,1)=l(1,0)=1$ where both false negatives and false positives are equally penalized and $\textrm{dist}(\vec{x},\vec{x}')=\lVert\vec{x}-\vec{x}'\rVert^2$, $\forall\vec{x},\vec{x}'\in\mathcal{X}$. As a common practice, we assume a thresholding classifier $f(\vec{x})=\mathbf{1}_{h(\vec{x})\geq0}$ where $h:\mathcal{X}\to\mathbb{R}$ is a continuous scoring function. We will also limit our attention to a type of {\em threshold abstention function}, given as follows.

\begin{definition}\label{as:abstention_threshold}
The principal's abstention decision depends only on the magnitude of the scoring function's output, $h(\hat{\vec{x}})$. A non-negative threshold $T \ge 0$ is specified, and abstention occurs if $|h(\hat{\vec{x}})| < T$. The abstention function $r(\hat{\vec{x}})$ thus takes the following form: 
\begin{equation}
\label{eq:abstention_function}
r(\hat{\vec{x}})=
\begin{cases}
1, & \text{if } |h(\hat{\vec{x}})| \geq T \\
0, & \text{if } |h(\hat{\vec{x}})| < T \\
\end{cases}.
\end{equation}
\end{definition}

This is a simple but effective choice of abstention since the scoring function's magnitude usually reflects the confidence of the classifier, which directly relates to the conditional loss at $\vec{x}$. For example, the scoring function for Bayes' classifier, i.e., $h(\vec{x})=p(y=1|\vec{x})-p(y=0|\vec{x})$, where $p(y|\vec{x})$ is the posterior distribution, is more ``uncertain'' about the data when $h(\vec{x})\approx 0$ and the loss, represented by the Bayes risk, decreases monotonically as $h(\vec{x})$ grows further away from zero.
With this definition, the principal's problem reduces to finding the optimal threshold $T^*$ that minimizes its expected loss $L(T)$.

Consider a one-dimensional example where $\mathcal{X}=[-2,2]$ and $x \sim \text{Unif}[-2,2]$; the label is given deterministically by $y=\textrm{sign}(x)$.
We fix the scoring function as the optimal scoring function without strategically manipulative agents, i.e.,
We apply the threshold abstention function defined in Definition \ref{as:abstention_threshold}, which in this specific case simplifies to:
\begin{eqnarray}
r(x)=
\begin{cases}
1, & \text{if } |x| \geq T \\
0, & \text{if } |x| < T \\
\end{cases}
\end{eqnarray}
The agent will either best respond with its true feature $x$ or a manipulate to reach the target of $x=T$, as this is the minimum level that achieves a positive classification and avoids abstention. This leads to the following result.

\begin{proposition}
\label{prop: agent_best_response_case_study}
For a fixed linear classifier $f(x) = x$, a threshold abstention rule with threshold $T$, and an agent with true feature $x$ and manipulation cost factor $K = \frac{1}{\sqrt{\gamma}}$, the agent's best response $\hat{x}$ is given by:
\begin{eqnarray}
\label{eq:agent_br_eqn}
\hat{x}(x)=
\begin{cases}
T, & \text{if } x \in (\max(-2, T-K), T) \\
x, & \text{otherwise}
\end{cases}
\end{eqnarray}
\end{proposition}

\subsection{The Principal's Optimal Abstention Threshold}
\label{ssec: dm_opt_T_case_study}

Given Definition \ref{as:abstention_threshold}, the principal's loss can be parameterized and rewritten as
$L(T) = \mathbb{E}_x[l(x, \hat{x}, T)]$, where $l(x, \hat{x}, T) := c \cdot \mathbf{1}_{|\hat{x}|<T} + \mathbf{1}_{|\hat{x}|>T} \cdot ( \mathbf{1}_{\{x<0\}} \cdot \mathbf{1}_{\{\hat{x}>0\}} + \mathbf{1}_{\{x>0\}} \cdot \mathbf{1}_{\{\hat{x}<0\}} )$. The function $l(x, \hat{x}, T)$ depends on the agent's manipulated feature $\hat{x}$ (given in \cref{prop: agent_best_response_case_study}) and takes values from $\{c,0,1\}$.

We denote the optimal abstention threshold when agents are strategic as $\bar{T}^*$. The following result shows the optimal $\bar{T}^*$ and the associated minimum principal's loss.

\begin{proposition}
\label{prop:dm_br}
Considering all cases of $K$ and $c$, the principal's optimal abstention threshold $\bar{T}^*$ and its corresponding minimum loss are given by:
\begin{eqnarray} 
\bar{T}^* &=& \begin{cases}
\min(K, 2) & \text{if } c < 0.5, 0 < K \leq 4 \\
\left[\frac{K}{2}, \min(K, 2)\right] & \text{if } c = 0.5, 0 < K \leq 4 \\
\frac{K}{2} & \text{if } c > 0.5, 0 < K \leq 4 \\
[0, 2] & \text{if } K > 4
\end{cases} \label{eq:dm_br_T}\\
L(\bar{T}^* ) &=& \begin{cases}
\frac{cK}{4} & \text{if } c < 0.5, 0 < K < 2 \\
\frac{1}{4}[K-2+c(4-K)] & \text{if } c < 0.5, 2 \leq K \leq 4 \\
\frac{K}{8} & \text{if } c \geq 0.5, 0 < K \leq 4 \\ 
\frac{1}{2} & \text{if } K > 4 \label{eq:dm_br_L}
\end{cases}
\end{eqnarray}
\end{proposition}

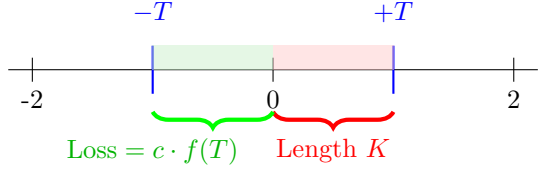
\begin{wrapfigure}{!htbp}{0.65\textwidth}
    \begin{minipage}{0.65\textwidth} 
    \vspace{-10pt}
\centering
\begin{tikzpicture}[x=2cm,y=2cm,scale=0.8]
\def\Tval{1}
\def\Kval{1}

\draw[-] (-2.2,0) -- (2.2,0);

\foreach \x in {-2,0,2}
\draw (\x,0.1) -- (\x,-0.1) node[below] {\x};

\draw[blue, thick] (-\Tval,-0.2) -- node[above,yshift=5mm] {$-T$} (-\Tval,0.2);
\draw[blue, thick] (\Tval,-0.2) -- node[above,yshift=5mm] {$+T$} (\Tval,0.2);

\fill[red!20, opacity=0.5] (0,0) rectangle (\Tval,0.2);
\fill[green!70!black!20, opacity=0.5] (-\Tval,0) rectangle (0,0.2);

\draw[red, ultra thick, decorate, decoration={brace,amplitude=6pt, mirror}]
(0,-0.35) -- (\Kval,-0.35) node[midway,below=6pt,red] {$\text{Length } K$};

\draw[green, ultra thick, decorate, decoration={brace,amplitude=6pt}]
(0,-0.35) -- (-\Kval,-0.35) node[right,below=6pt,green!70!black] {$\text{Loss} =c \cdot f(T)$};
\end{tikzpicture}
\caption{Illustration of expected loss contributions for the uniform distribution case study when $c < 0.5$ and $0 < K < 2$.}
\label{fig:number_line_1}
\vspace{-20pt}
\end{minipage}
\end{wrapfigure}

To provide intuition for the optimal threshold $\bar{T}^*$ and minimum loss $L(\bar{T}^*)$, we examine the case where $c < 0.5$ and $0 < K < 2$. In this scenario, the optimal threshold is $\bar{T}^*=K$ and the minimum loss is $L(\bar{T}^*)=\frac{cK}{4}$. Figure \ref{fig:number_line_1} visualizes the regions contributing to the expected loss $L(T)$. The horizontal axis represents the true feature $x\in[-2,2]$, uniformly distributed with density $p(x)=\frac{1}{4}$. The vertical blue lines mark the abstention threshold $T$ and $-T$.

The green shaded region, spanning from $-T$ to $0$, illustrates instances where the principal incurs an abstention loss. For a true feature $x$ in this interval, the agent does not manipulate. The condition $|x|<T$ triggers the abstention rule $r(x)=0$, leading to a cost $c$ for the decision maker. Given the uniform density, this region's contribution to the total expected loss is $\frac{cT}{4}$.

The red shaded region, spanning from $0$ to $T$, represents true feature values $x$ for which agents manipulate to $\hat{x}=T$, as per Proposition \ref{prop: agent_best_response_case_study}. When the agent submits $\hat{x}=T$, the classifier predicts $\hat{y}=1$, which correctly aligns with the true label, resulting in zero misclassification loss. Furthermore, since $|\hat{x}|=T \ge T$, no abstention loss is incurred. Therefore, this region contributes zero loss to the decision maker's expected loss $L(T)$. Thus, the loss function in this case yields $\bar{T}^* = K$ and an optimal loss $L(\bar{T}^*) = \frac{cK}{4}$.

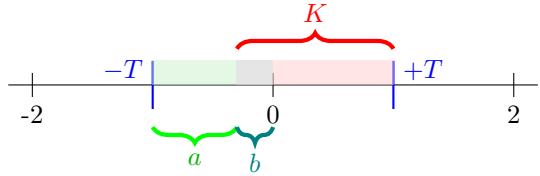
\begin{wrapfigure}{!htbp}{0.65\textwidth}
    \begin{minipage}{0.65\textwidth} 
    \vspace{-10pt}
\centering
\begin{tikzpicture}[x=2cm,y=2cm,scale=0.8]
\def\Tval{1}
\def\Kval{1}

\draw[-] (-2.2,0) -- (2.2,0);

\foreach \x in {-2,0,2}
\draw (\x,0.1) -- (\x,-0.1) node[below] {\x};

\draw[blue, thick] (-\Tval,-0.2) -- node[left,yshift=2mm] {$-T$} (-\Tval,0.2);
\draw[blue, thick] (\Tval,-0.2) -- node[right,yshift=2mm] {$+T$} (\Tval,0.2);

\fill[red!20, opacity=0.5] (-0.3,0) rectangle (\Tval,0.2);
\fill[green!70!black!20, opacity=0.5] (-\Tval,0) rectangle (-0.3,0.2);
\fill[teal!20, opacity=0.5] (-0.3,0) rectangle (0,0.2);

\draw[red, ultra thick, decorate, decoration={brace,amplitude=6pt}]
(-0.3,0.3) -- (\Kval,0.3) node[midway,above=6pt,red] {$K$};

\draw[green, ultra thick, decorate, decoration={brace,amplitude=6pt}]
(-0.3,-0.35) -- (-\Tval,-0.35) node[midway,below=6pt,green!70!black] {$a$};

\draw[teal, ultra thick, decorate, decoration={brace,amplitude=6pt}]
(0,-0.35) -- (-0.3,-0.35) node[midway,below=6pt,teal] {$b$};
\end{tikzpicture}
\caption{Illustration of expected loss contributions when $c = 0.5$ and $0 < K < 2$. In this scenario, the optimal threshold $\bar{T}^*$ is a range $[\frac{K}{2}, K]$.}
\label{fig:number_line_2}
\vspace{-20pt}
    \end{minipage}
\end{wrapfigure}

Interestingly, there are cases where $\bar{T}^*$ can take a range of values. For instance, when the abstention cost $c = 0.5$ and the manipulation cost factor $K \in (0, 2)$, the optimal threshold $\bar{T}^*$ is $[\frac{K}{2}, K]$, with a constant minimum expected loss of $L(\bar{T}^*) = \frac{K}{8}$.
Figure \ref{fig:number_line_2}'s red shaded region shows where agents manipulate to $\hat{x}=T$. According to Proposition \ref{prop: agent_best_response_case_study} and the fact that $|T-K|< 2$, the manipulation region simplifies from $(\max(-2,T-K),T)$ to $(T-K,T)$, with a total length of $K$. Within this region, only agents with $x < 0$ are misclassified, contributing a misclassification loss of $\frac{K-T}{4}$ (region $b$).
The green shaded region, of length $2T-K$, represents cases where the principal incurs an abstention loss. These agents do not manipulate, contributing an abstention loss of $\frac{2T-K}{8}$. 
The total expected loss is calculated by summing both red and green regions, resulting in $\frac{K}{8}$. Notice that this is independent of $T$, illustrating the non-uniqueness of $\bar{T}^*$.

\subsection{The Equilibrium Without Abstention} \label{sec:no_abstention_summary}

Next, we analyze the equilibrium outcome when the decision maker does not have the ability to abstain from making a decision. The setup remains consistent with our primary model in the case study, but without the abstention function $r(x)$.
In this scenario, an agent's utility simplifies to $U(\hat{x}|x) = \mathbf{1}_{\hat{x} \geq 0} - \gamma(\hat{x}-x)^2$, reflecting their sole goal of achieving positive classifier prediction. Clearly, the agent's best response $\hat{x}(x)$ under this utility function is given by:
$$
\hat{x}(x) = \begin{cases}
0 & \text{if } -K \leq x < 0 \\
x & \text{otherwise}
\end{cases}
$$
where $K = 1/\sqrt{\gamma}$ represents the manipulation cost factor. The decision maker's loss $L$ is then the expected misclassification loss, defined as $L = \mathbb{E}_x[l(x,\hat{x})]$, where $l(x,\hat{x}) = \mathbf{1}_{x<0} \cdot \mathbf{1}_{\hat{x} \ge 0} + \mathbf{1}_{x>0} \cdot \mathbf{1}_{\hat{x} < 0}$. Given the agent's best response, misclassification occurs only when agents with a true negative feature ($x<0$) successfully manipulate to achieve a positive classification ($\hat{x} \ge 0$).

\begin{proposition}
\label{prop:no_abs_loss}
The expected loss $L_{no\_abstention}$ for the decision maker in the absence of abstention is piecewise defined based on the value of $K$:
\begin{equation} \label{eq:exp-noabs}
L_{no\_abstention} = \begin{cases}
\frac{K}{4} & \text{if } K \le 2 \\
\frac{1}{2} & \text{if } K > 2
\end{cases}
\end{equation}

\end{proposition}

This result highlights how, without the abstention mechanism, the decision maker faces a loss due to strategic manipulation directly proportional to the manipulation cost factor $K$ up to a certain point.
Notably, Eq. \eqref{eq:dm_br_L} is no greater than Eq. \eqref{eq:exp-noabs} for all parameter cases, justifying the value of the abstention mechanism in mitigating the decision maker's loss in the presence of strategic manipulation as shown in Theorem \ref{thm:benefit_abstention}.

\subsection{Expected Manipulation of Unqualified Agents} \label{sec:expected_manipulation_summary}
Beyond analyzing the principal's loss, we also examine the expected amount of manipulation by unqualified agents at equilibrium. With abstention, agents who are qualified may have an incentive to manipulate in order to be positively classified and not abstained by the classifier. However, these agents' manipulation does impact the principal's loss. Thus, we focus on the \textit{expected manipulation by unqualified agents}, who, otherwise would have been classified as negative, manipulate to receive a positive outcome. We define expected manipulation as the average absolute difference between an agent's true feature $x$ and their manipulated feature $\hat{x}$ under best response, $D = \mathbb{E}[|x - \hat{x}(x)|]$.

\vspace{-10pt}
\subsubsection{Expected Manipulation Without Abstention:}
In the model without abstention, the expected manipulation, $E_{\text{no\_abstention}}$, is solely a function of the manipulation cost factor $K$. Agents manipulate when their true feature $x$ falls within $[-K, 0)$. 
The expected manipulation increases quadratically in $K$ until saturation:
\[
E_{\text{no\_abstention}} = \begin{cases}
\frac{K^2}{8} & \text{if } 0 < K \leq 2 \\
\frac{1}{2} & \text{if } K > 2
\end{cases}
\]

\vspace{-10pt}
\subsubsection{Expected Manipulation With Abstention:}
In the model with abstention, qualified agents can also manipulate to be positively classified. To only consider unqualified agents, we restrict our integration over $[-2, 0)$. The expected manipulation, $E_{\text{with\_abstention}}$, varies across different $K$ values:

\[
E_{\text{with\_abstention}} = \begin{cases}
0 & \text{if } 0 < K \le 2 \text{ and } c < 0.5 \\
\frac{K^2}{8} - \frac{1}{2} & \text{if } 2 \le K \le 4 \text{ and } c < 0.5 \\
\frac{3K^2}{32} & \text{if } 0 < K \le 4 \text{ and } c \ge 0.5 \\
\frac{\bar{T^*}}{2} + \frac{1}{2} & \text{if } K > 4 \quad (\text{where } 0 \le \bar{T^*} \le 2)
\end{cases}
\]

\vspace{-10pt}
\subsubsection{Comparing $E_{\text{no\_abstention}}$ and $E_{\text{with\_abstention}}$:}

The comparison reveals that the expected manipulation by unqualified agents is highly dependent on both the manipulation cost $\gamma$ (represented by $K = \frac{1}{\sqrt{\gamma}}$) and the abstention cost $c$. Notably, abstention often reduces manipulation by unqualified agents: for high $\gamma$ (low $K$, $0 < K \le 2$), abstention consistently leads to lower expected manipulation by unqualified agents, regardless of the abstention cost $c$. This represents a primary benefit of adopting an abstention option. However, the comparison also reveals a potential for increased manipulation in certain cases. Specifically, for intermediate $\gamma$ ($2 < K \le 4$), abstention can either reduce or increase manipulation. This implies that simply introducing an abstention mechanism does not guarantee a reduction in this type of unqualified manipulation across all parameters. Furthermore, for very low $\gamma$ (high $K$, $K > 4$), the introduction of abstention can lead to either equal or increased manipulation by unqualified agents, depending on the principal's chosen threshold. This comparison indicates that the effectiveness of abstention in deterring manipulation is highly dependent on the parameters of the system. While the proportion of manipulative unqualified agents is reduced, the average amount of manipulation, as measured by $E_{\textrm{no\_abstention}}$ and $E_{\textrm{with\_abstention}}$ can be the opposite. Overall, abstention can still serve as a valuable tool for deterring manipulation when manipulation costs are substantial.

%% file: simulation.tex
\section{Simulation Results}
\label{sec:sim}

\subsection{Simulation Setup and Impact of System Parameters}
\label{sim_setup}
We simulate the Stackelberg game from Section \ref{sec: uniform_example}, where the decision maker sets an abstention threshold $T$ and agents respond strategically. We compare the unconstrained setting (truthful agents, optimal threshold $T^*$) to the constrained setting (manipulating agents, optimal threshold $\bar{T}^*$). For each threshold $T \in [0.01, 2.0]$ with step size 0.01, we draw $100,000$ features $x\sim\textrm{Unif}[-2,2]$ and compute labels via $y = \textbf{1}_{x+\epsilon > 0}$, with $\epsilon \sim \mathcal{N}(0,\sigma)$. In the constrained case, agents either report $x$ or manipulate to $T$, following Proposition \ref{prop: agent_best_response_case_study}. We compute the pointwise loss and average over samples to estimate $L(T)$. Grid search yields optimal $T^*$ and $\bar{T}^*$. We then study how these values change as we vary $\sigma$, $\gamma$, and $c$, while fixing the other two to default values of $\gamma = 0.4444$, $c = 0.3$, $\sigma = 0.5$.

\begin{figure}[t]
\centering
\begin{subfigure}[b]{0.32\textwidth}
\includegraphics[width=\textwidth]{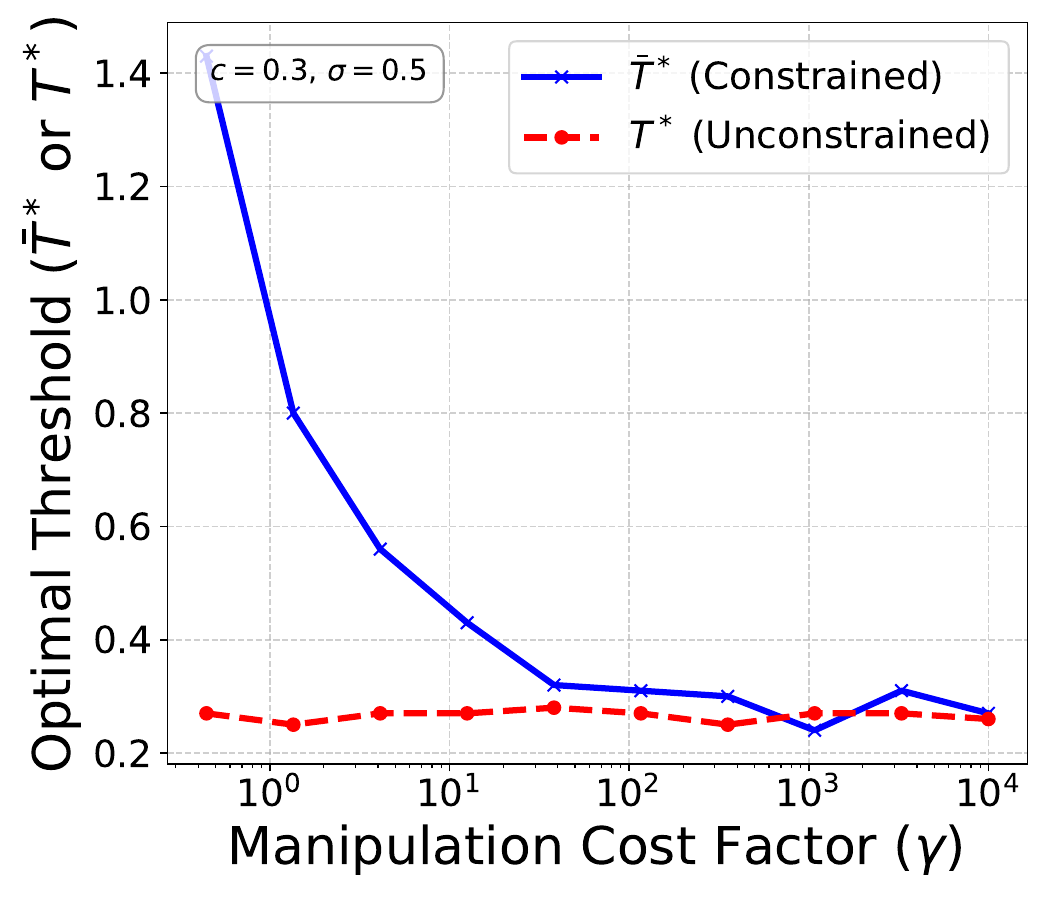}
\caption{Manipulation cost $\gamma$}
\label{fig:T_vs_gamma}
\end{subfigure}
\hfill
\begin{subfigure}[b]{0.32\textwidth}
\includegraphics[width=\textwidth]{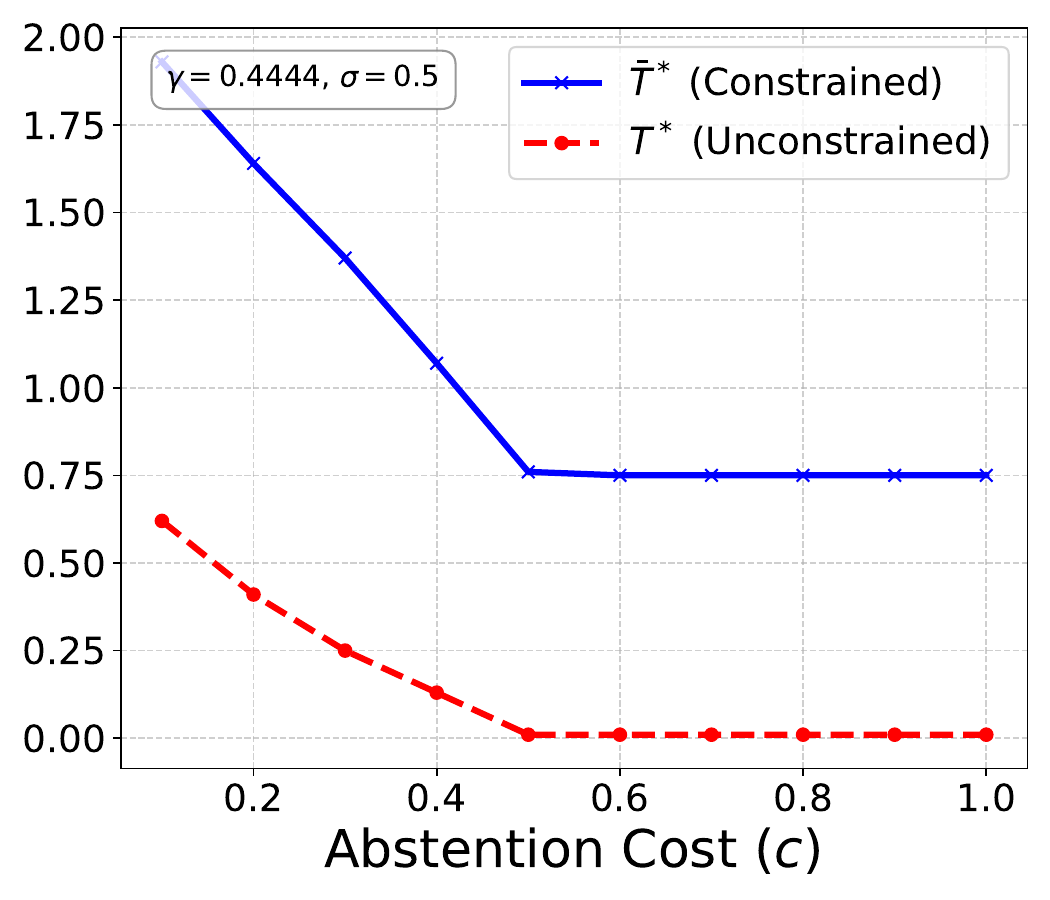}
\caption{Abstention cost $c$}
\label{fig:T_vs_c}
\end{subfigure}
\hfill
\begin{subfigure}[b]{0.32\textwidth}
\includegraphics[width=\textwidth]{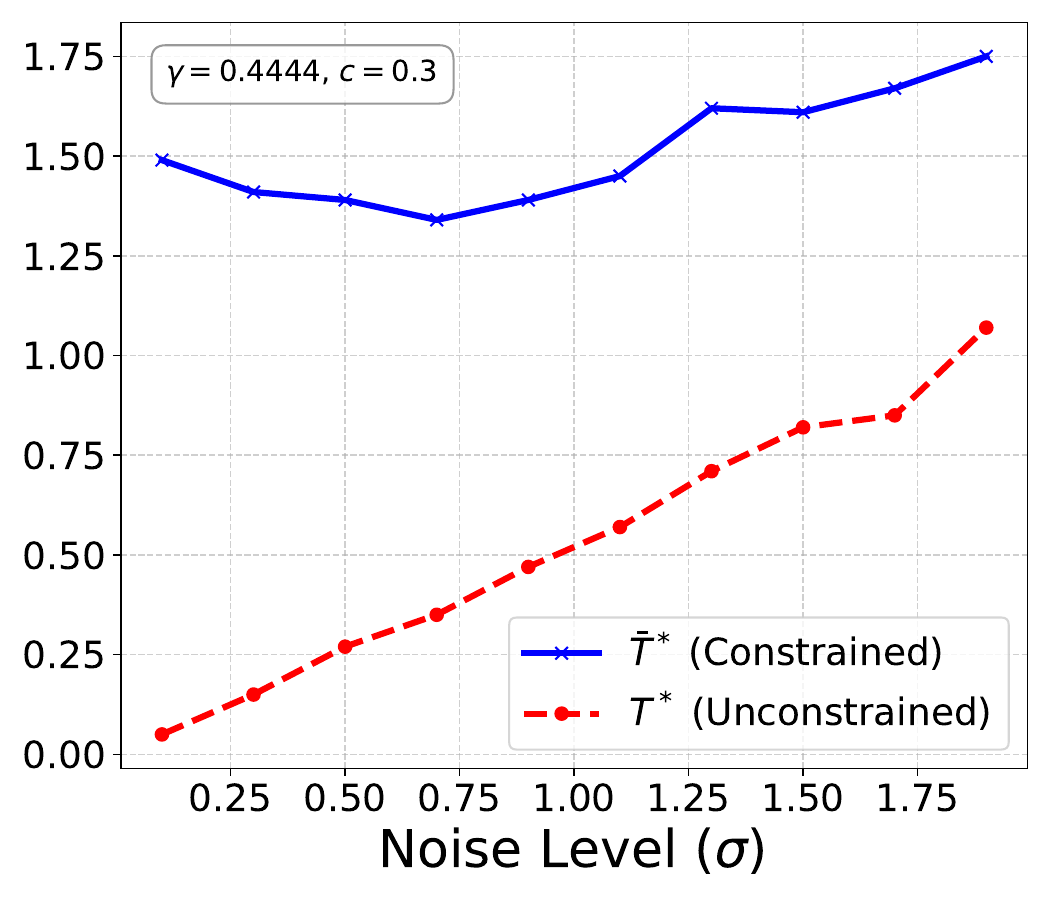}
\caption{Noise level $\sigma$}
\label{fig:T_vs_sigma}
\end{subfigure}
\caption{Optimal thresholds $\bar{T}^*$ and $T^*$ under varying parameters.}
\label{fig:constr_vs_unconstr}
\vspace{-20pt}
\end{figure}

Figure \ref{fig:T_vs_gamma} shows that in the constrained case, the optimal threshold $\bar{T}^*$ decreases as the agent's manipulation cost $\gamma$ increases (with abstention cost fixed at $c=0.3$). When manipulation is cheap, the principal sets a high $\bar{T}^*$ to deter manipulation. As $\gamma$ increases, manipulation becomes costlier, allowing the principal to lower $\bar{T}^*$. In contrast, the unconstrained threshold $T^*$ remains constant, and $\bar{T}^* \rightarrow T^*$ as manipulation becomes prohibitively costly.

Figure \ref{fig:T_vs_c} shows how the optimal thresholds vary with abstention cost $c$. As $c$ increases, $T^*$ and $\bar{T}^*$ generally decrease or stay flat, reflecting the trade-off between abstention and misclassification. Low $c$ allows the principal to set a higher threshold and abstain more. As $c$ grows, abstention becomes costlier, encouraging lower thresholds. Interestingly, both constrained and unconstrained thresholds exhibit a turning point around the same $c=0.5$, after which the abstention cost becomes even worse than a random guess. The post-turning position of the constrained threshold is remarkably higher, reflecting the abstention's hedging effect against agents' manipulation.

Figure \ref{fig:T_vs_sigma} shows how optimal thresholds vary with noise level $\sigma$. In the unconstrained case, rising noise increases uncertainty in true labels, prompting the principal to raise $T^*$ to avoid misclassification. In the constrained case, $\bar{T}^*$ first dips, because moderate noise weakens the reliability of manipulation, allowing a lower threshold. But at high noise levels, the principal again raises $\bar{T}^*$ to reduce their misclassification risk.

Figure \ref{fig:constr_vs_unconstr} confirms that the constrained threshold $\bar{T}^*$ is consistently higher than the unconstrained $T^*$, despite sampling noise. This matches Theorem \ref{theorem:no-larger-accept}, which states that principals facing strategic agents abstain more. The elevated threshold deters manipulation by unqualified agents, who would otherwise exploit a lower threshold to gain acceptance, increasing misclassification risk. Thus, {\em the constrained principal raises $\bar{T}^*$ to guard against gaming.}

\subsection{Mitigating Strategic Harm through Optimal Abstention}
Strategic agents manipulate inputs to maximize their utility, often at the principal's expense. We define the resulting increase in expected loss as the "harm" caused by strategic behavior, measured relative to a non-strategic baseline. This harm is quantified in two settings. When the principal cannot abstain (i.e., $T = 0$), the harm is defined as $H_{\text{no abstention}} = L(T=0, \text{strategic}) - L(T=0, \text{non-strategic})$. When the principal can optimally abstain using threshold $\bar{T}^*$, we define $H_{\text{abstention}} = L(\bar{T}^*, \text{strategic}) - L(T^*, \text{non-strategic})$. The effectiveness of abstention in reducing harm is captured by the difference: $$\Delta H = \{L(\bar{T}^*, \text{strat.}) - L(T^*, \text{non-strat.})\} - \{ L(T=0, \text{strat.}) - L(T=0, \text{non-strat.}) \}$$

\begin{figure}[t]
\centering
\begin{subfigure}[b]{0.32\textwidth}
\includegraphics[width=\textwidth]{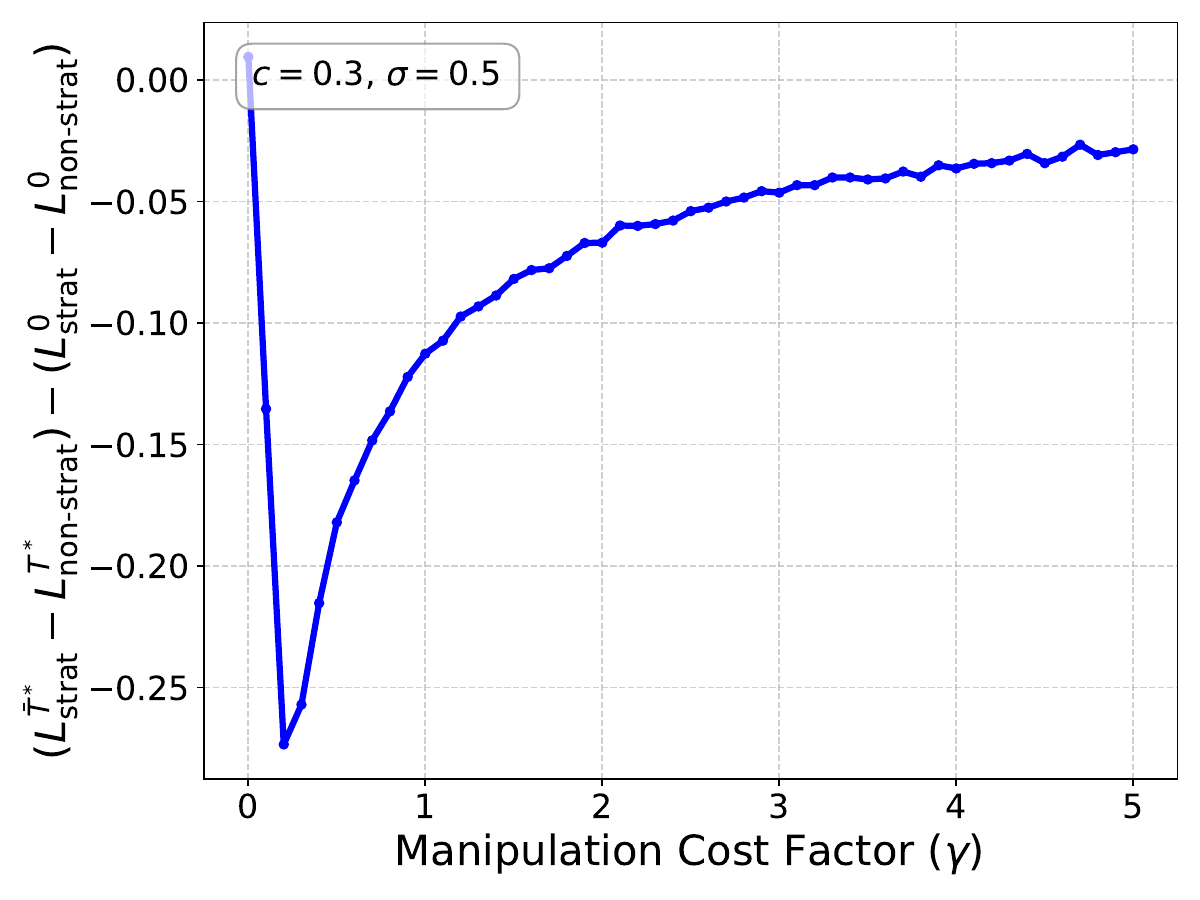}
\caption{Manipulation cost $\gamma$}
\label{fig:diff_vs_gamma}
\end{subfigure}
\hfill
\begin{subfigure}[b]{0.32\textwidth}
\includegraphics[width=\textwidth]{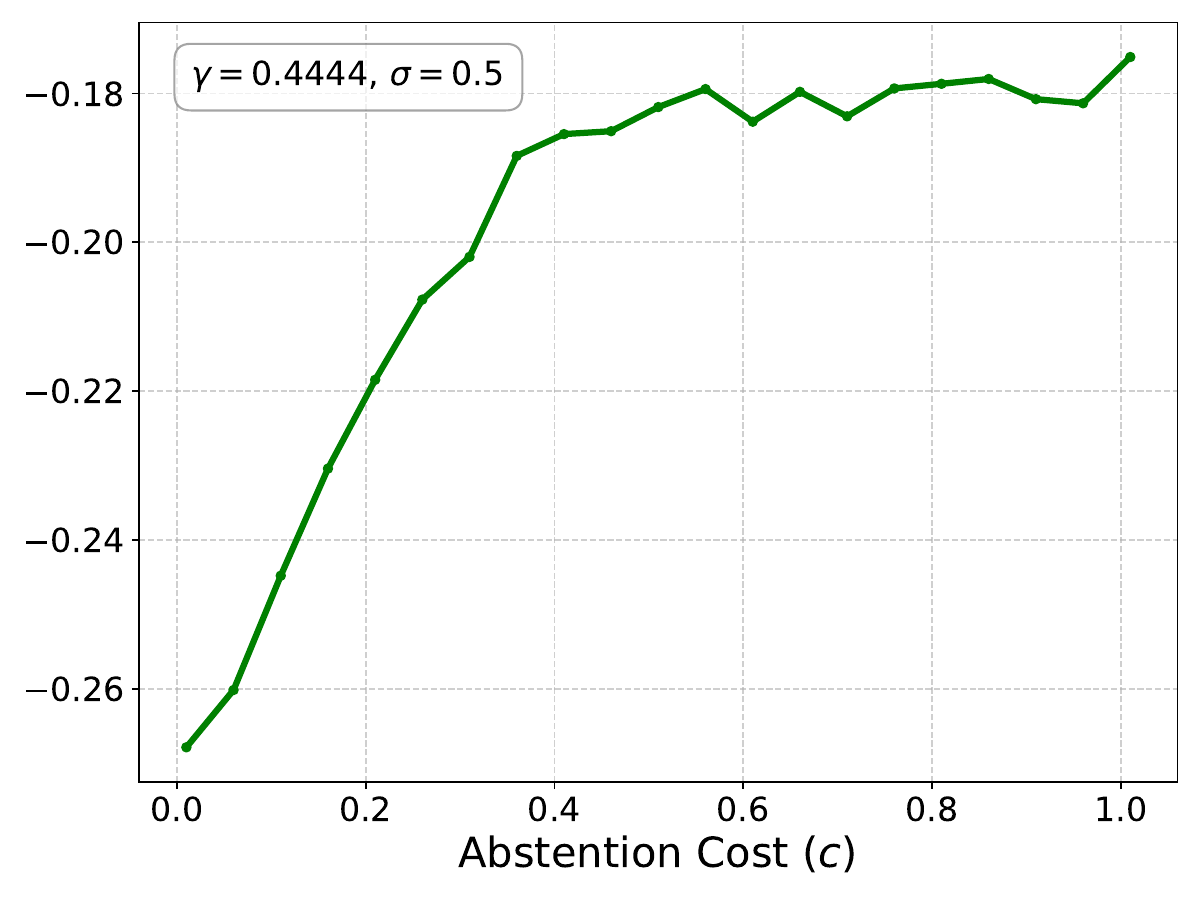}
\caption{Abstention cost $c$}
\label{fig:diff_vs_c}
\end{subfigure}
\hfill
\begin{subfigure}[b]{0.32\textwidth}
\includegraphics[width=\textwidth]{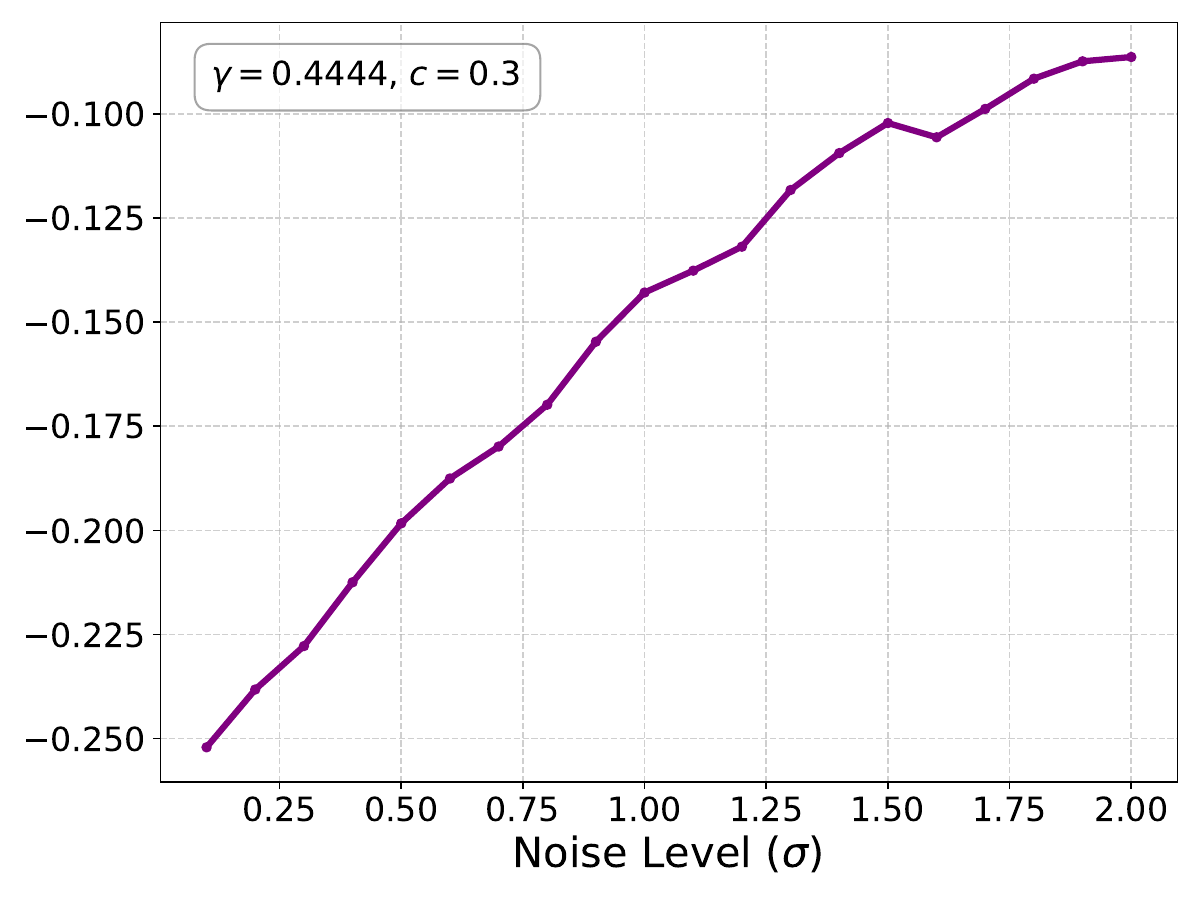}
\caption{Noise level $\sigma$}
\label{fig:diff_vs_sigma}
\end{subfigure}
\caption{Harm reduction via optimal abstention across system parameters.}
\label{fig:reduce_harm}
\vspace{-20pt}
\end{figure}

A negative $\Delta H$ indicates that abstention reduces the harm caused by strategic agents. We evaluate this effect by sweeping key parameters, using the simulation setup from Section~\ref{sim_setup}. When $\gamma$ is very small, manipulation is effectively free, giving agents near-unlimited capacity to alter their features. In this case, abstention has little effect on the principal's loss, resulting in an almost zero leftmost value of Figure~\ref{fig:diff_vs_gamma}. As $\gamma$ increases, manipulation incurs real cost, and agents face a tradeoff between benefit and expense. Here, optimal abstention becomes highly effective, yielding a negative $\Delta H$ and significantly reducing strategic harm. As $\gamma$ continues to rise, manipulation becomes prohibitive, and strategic behavior converges with non-strategic behavior. In this case, both $H_{\textrm{abstention}}$ and $H_{\textrm{no abstention}}$ converge to zero as $\gamma$ increases. Thus, $\Delta H$ increases with $\gamma$.

Figure~\ref{fig:diff_vs_c} shows that $\Delta H$ generally increases with abstention cost $c$, as one would expect intuitively. $H_{\text{no abstention}}$ remains constant, while as $c$ increases, the principal abstains less, leading to higher $H_{\text{abstention}}$ and a subsequent rise in $\Delta H$. Figure~\ref{fig:diff_vs_sigma} shows that $\Delta H$ increases as noise $\sigma$ increases. Greater noise reduces both label predictability and the effectiveness of strategic manipulation. In highly noisy settings, the behaviors of strategic and non-strategic agents become indistinguishable, so abstention yields little additional benefit and $\Delta H$ converges toward zero.

%% file: appendix.tex
\appendix
\setlength{\abovedisplayskip}{3pt} 
\setlength{\belowdisplayskip}{2pt} 

\section{Supplementaries of Section \ref{sec:fixed_f_analysis}}
\subsection{Proof of Theorem \ref{theorem:negative-coincides}}
Since the classifier $f$ is fixed, we simplify the notation of the principal's expected loss to $L(r)$ in this proof to emphasize the optimizing variable. When $\bar{r}^*$ is the optimal constrained abstention function, it must satisfies $L(\bar{r}^*)\leq L(r)$, $\forall r$. The minimum expected loss can be written as
\begin{equation}\label{eq:L-decompose}
\begin{aligned}
    L(\bar{r}^*) = &\;\int_{\vec{x},y}\left[l(f(\hat{\vec{x}}),y)\bar{r}^*(\hat{\vec{x}}) + c(1-\bar{r}^*(\hat{\vec{x}}))\right]p(\vec{x},y)\;dyd\vec{x} \\
    =&\;\int_{\hat{\vec{x}},y}\left[l(f(\hat{\vec{x}}),y)\bar{r}^*(\hat{\vec{x}}) + c(1-\bar{r}^*(\hat{\vec{x}}))\right]q(\hat{\vec{x}},y)\;dyd\hat{\vec{x}} \\
    =&\;\underbrace{\int_{\hat{\vec{x}},y:f(\hat{\vec{x}})=0}\left[l(f(\hat{\vec{x}}),y)\bar{r}^*(\hat{\vec{x}}) + c(1-\bar{r}^*(\hat{\vec{x}}))\right]q(\hat{\vec{x}},y)\;dyd\hat{\vec{x}}}_{A} \\
    &\quad+ \underbrace{\int_{\hat{\vec{x}},y:f(\hat{\vec{x}})=1}\left[l(f(\hat{\vec{x}}),y)\bar{r}^*(\hat{\vec{x}}) + c(1-\bar{r}^*(\hat{\vec{x}}))\right]q(\hat{\vec{x}},y)\;dyd\hat{\vec{x}}}_{B}
\end{aligned}
\end{equation}
where $q(\hat{\vec{x}},y)$ denote the joint density of the post-response features and the true label, and $A$ (resp. $B$) represents the expected loss incurred from negatively (resp. positively) labeled data points. Notice that the post-response joint density can be decomposed by
\begin{equation}\label{eq:post-join}
    q(\hat{\vec{x}},y) = \int_{\vec{x}}p(y|\vec{x})q(\hat{\vec{x}}|\vec{x})p(\vec{x})\;d\vec{x}
\end{equation}
where we utilized the fact that $y$ is independent of $\hat{\vec{x}}$ given true attribute $\vec{x}$ since $\hat{\vec{x}}$ can be written as a function of $\vec{x}$. Let $N_{f,\bar{r}^*}(\hat{\vec{x}})$ be the set of feature vectors in $\mathcal{X}$ whose post-response features under $f$ and $\bar{r}^*$ equals $\hat{\vec{x}}$. Formally, when $f(\hat{\vec{x}})=0$ or $\bar{r}^*(\hat{\vec{x}})=0$, we have 
\begin{equation}\label{eq:set-negative}
N_{f,\bar{r}^*}(\hat{\vec{x}})=\begin{cases}
    \emptyset & \exists\vec{x}\in B_{\frac{1}{\gamma}}(\hat{\vec{x}})\; \textrm{s.t.}\; f(\vec{x})=\bar{r}^*(\vec{x})=1\\ 
    \{\hat{\vec{x}}\} & \textrm{otherwise}\\
\end{cases}
\end{equation}
where $B_{\frac{1}{\gamma}}(\hat{\vec{x}}):=\{\vec{x}\in\mathcal{X}:\textrm{dist}(\vec{x},\hat{\vec{x}})\leq \frac{1}{\gamma}\}$ is the $\frac{1}{\gamma}$-ball around $\hat{\vec{x}}$ under the distance measure $\textrm{dist}(\cdot,\cdot)$; when $f(\hat{\vec{x}})=\bar{r}^*(\hat{\vec{x}})=1$, we have 
\begin{multline}\label{eq:set-positive}
    N_{f,\bar{r}^*}(\hat{\vec{x}}) = \{\hat{\vec{x}}\}\cup\Big\{\vec{x}\in B_{\frac{1}{\gamma}}(\hat{\vec{x}}):
    f(\vec{x})\bar{r}^*(\vec{x}) = 0\textrm{ and }\\ \forall \vec{z}\in\mathcal{X},\;f(\vec{z})=\bar{r}^*(\vec{z})=1\implies \textrm{dist}(\vec{x},\hat{\vec{x}})\leq \textrm{dist}(\vec{z},\hat{\vec{x}})   \Big\}.
\end{multline}
The second set includes all features at which the agent would strategically manipulate to $\hat{\vec{x}}$, in best response to the classifier $f$ and abstention $\bar{r}^*$.

Using the notations above, the post-response distribution can be written as 
$
    q(\hat{\vec{x}}|\vec{x}) = \mathbf{1}_{\vec{x}\in N_{f,\bar{r}^*}(\hat{\vec{x}})}
$.
We next show that the values of $\bar{r}^*({\vec{x}})$ for negatively classified ${\vec{x}}$ does not influence the set $N_{f,\bar{r}^*}(\hat{\vec{x}})$ such that $f(\hat{\vec{x}})=1$, which further implies the term $B$ is invariant to abstention choices on negatively classified data points because $B$'s integrand can only depend on other feature vectors through $q(\hat{\vec{x}},y)$, which, according to Eq. \eqref{eq:post-join}, depends only on $q(\hat{\vec{x}}|\vec{x})$. 

Suppose ${\vec{x}},\hat{\vec{x}}\in\mathcal{X}$ such that $f(\vec{x})=0$ and $f(\hat{\vec{x}})=1$. If $\bar{r}^*(\hat{\vec{x}})=0$, it's directly follows from Eq. \eqref{eq:set-negative} that $N_{f,\bar{r}^*}(\hat{\vec{x}})$ is independent of $\bar{r}^*(\vec{x})$. If $\bar{r}^*(\hat{\vec{x}})=1$, changing the value of $\bar{r}^*(\vec{x})$ does not change the either condition in Eq. \eqref{eq:set-positive}, implying $N_{f,\bar{r}^*}(\hat{\vec{x}})$ is independent of $\bar{r}^*(\vec{x})$ as well. Thus, we conclude that $B$ is independent of $\bar{r}^*(\vec{x})$ for all negatively classified $\vec{x}$.

According to Eq. \eqref{eq:post-join} and \eqref{eq:set-negative}, the term $A$ can be decomposed into
\begin{multline}
    A=0+\int_{\hat{\vec{x},y}:\substack{f(\hat{\vec{x}})=0\\ N_{f,\bar{r}^*}(\hat{\vec{x}})=\{\hat{\vec{x}}\}}}\left[l(f(\hat{\vec{x}}),y)\bar{r}^*(\hat{\vec{x}}) + c(1-\bar{r}^*(\hat{\vec{x}}))\right]q(\hat{\vec{x}},y)\;dyd\hat{\vec{x}}
\end{multline}
where the first term is due to zero post-response density at those features. Besides, when $N_{f,\bar{r}^*}(\hat{\vec{x}})=\{\hat{\vec{x}}\}$, we obtain $q(\hat{\vec{x}},y)=p(\hat{\vec{x}},y)$ by Eq. \eqref{eq:post-join} and thus $A$ becomes
\begin{multline*}
    A=\int_{\hat{\vec{x}},y:\substack{f(\hat{\vec{x}})=0\\ N_{f,\bar{r}^*}(\hat{\vec{x}})=\{\hat{\vec{x}}\}}}\left[l(f(\hat{\vec{x}}),y)\bar{r}^*(\hat{\vec{x}}) + c(1-\bar{r}^*(\hat{\vec{x}}))\right]p(\hat{\vec{x}},y)\;dyd\hat{\vec{x}} \\
    = \int_{\hat{\vec{x}}:\substack{f(\hat{\vec{x}})=0\\ N_{f,\bar{r}^*}(\hat{\vec{x}})=\{\hat{\vec{x}}\}}}\Big[L_f(\hat{\vec{x}})\bar{r}^*(\hat{\vec{x}})+c(1-\bar{r}^*(\hat{\vec{x}}))\Big]p(\hat{\vec{x}})\;dyd\hat{\vec{x}}.
\end{multline*}
Recall that the unconstrained abstention function $r^*$ is a pointwise minimizer provided the conditional loss $L_f$. We must also have $r^*$ on $\hat{\vec{x}}$ such that $f(\hat{\vec{x}})=0$ minimizes $A$. As a consequence, $\tilde{r}^*$, the optimal constrained $\bar{r}^*$ with negatively labeled points interpolated by $r^*$, is also an optimal constrained solution.\qed

\subsection{Proof of Theorem \ref{theorem:no-larger-accept}}

Suppose, on the contrary, that $\bar{r}^*> r^*$, i.e., $r^*(\vec{x})=1\implies \bar{r}^*(\vec{x})=1$ and $\exists\vec{x}$ s.t. $r^*(\vec{x})=0$, and $\bar{r}^*(\vec{x})=1$. 
In light of Theorem \ref{theorem:negative-coincides}, it suffices to focus only on the case that $\bar{r}^*\neq r^*$ only on $\{f(\vec{x})=1\}$.

Let $\hat{\vec{x}}(\vec{x}|r^*)$ (resp. $\hat{\vec{x}}(\vec{x}|\bar{r}^*)$) denote the agent's best response under abstention $r^*$ (resp. $\bar{r}^*$). Denote $g_{r}(\vec{x}):=\Big[l(f({\vec{x}}),y)r(\vec{{x}})+c(1-r({\vec{x}}))\Big]p(\vec{x},y)$. More agents in $\{f(\vec{x})=0\}$ would manipulate under $\bar{r}^*$ since more predictions are accepted. Thus,
\begin{equation}\label{eq:diff-L-zero}
    \int_{f(\vec{x})=0}g_{\bar{r}^*}(\vec{x})\;d\vec{x}-\int_{f(\vec{x})=0}g_{{r}^*}(\vec{x})\;d\vec{x} = \int_{\substack{f(\vec{x})=0 \\ \hat{\vec{x}}(\vec{x}|r^*)=\vec{x} \\ \hat{\vec{x}}(\vec{x}|\bar{r}^*)\neq \vec{x}}}G(\vec{x}|r^*,\bar{r}^*)p(\vec{x})\;d\vec{x}
\end{equation}
where 
\begin{equation*}
    G(\vec{x}|r^*,\bar{r}^*) = \begin{cases}
        l(1,0)p(y=0|\vec{x}) - l(0,1)p(y=1|\vec{x}) & r^*(\vec{x})=\bar{r}^*(\vec{x})=1 \\
        l(1,0)p(y=0|\vec{x}) - c & r^*(\vec{x})=\bar{r}^*(\vec{x})=0
    \end{cases}.
\end{equation*}
Since $r^*$ is the unconstrained optimal abstention, $l(1,0)p(y=0|\vec{x})\geq c$, $\forall\vec{x}$ in the integration. For the other case, as the agent with $\vec{x}$ does not manipulate under $r^*$, it must hold $f(\hat{\vec{x}}(\vec{x}|\bar{r}^*))=1$, $r^*(\hat{\vec{x}}(\vec{x}|\bar{r}^*))=0$, and $\bar{r}^*(\hat{\vec{x}}(\vec{x}|\bar{r}^*))=1$. By the informative assumption of $f$, we have $l(1,0)p(y=0|\vec{x})>l(1,0)p(y=0|\hat{\vec{x}}(\vec{x}|\bar{r}^*))>c$ where the last inequality is due to the optimality of the unconstrained abstention $r^*$. Similarly, $r^*(\vec{x})=1$ implies $l(0,1)p(y=1|\vec{x})\leq c$. Therefore, we conclude $G(\vec{x}|r^*,\bar{r}^*)\geq 0$ for all $\vec{x}$ integrated. 

Given the above results, we assume $f(\vec{x})=1$, $\forall \vec{x}$ without loss of generality and show that $L(\bar{r}^*)\geq L(r^*)$. We first claim that if $\hat{\vec{x}}(\vec{x}|r^*)\neq\vec{x}$ and $\bar{r}^*(\vec{x})=0$, then $\hat{\vec{x}}(\vec{x}|\bar{r}^*)\neq \vec{x}$. In other words, manipulative agents stay manipulative unless they are newly accepted as positive under $\bar{r}^*$. This can be easily seen from their utility function Eq. \eqref{eq:agent-utility} that, due to $\bar{r}^*>r^*$, $\hat{\vec{x}}(\vec{x}|r^*)$ still strategically dominates no manipulation under $\bar{r}^*$.
Define $M(r):=\{\hat{\vec{x}}(\vec{x}|r)\neq \vec{x}\}$ as the set of features at which the agent would manipulate in best response to the abstention function $r$.

Consider the subset of $M(r^*)$ defined as $A:=M(r^*)\cap\{\bar{r}^*(\vec{x})=0\}$. By the claim established in the previous paragraph, we have $A\subseteq M(\bar{r}^*)$. It also follows that $M(r^*)\setminus A\subseteq \{\bar{r}^*(\vec{x})=1,\;r^*(\vec{x})=0\}\subseteq M(\bar{r}^*)^c$ because $\hat{\vec{x}}(\vec{x}|r^*)\neq\vec{x}$ means it must receive negative outcome from $r^*$ and $\bar{r}^*(\vec{x})=1\implies \hat{\vec{x}}(\vec{x}|\bar{r}^*)=\vec{x}$ (assuming $f(\vec{x})=1$). 
In addition, we have $M(\bar{r}^*)\setminus A\subseteq \{\bar{r}^*(\vec{x})=r^*(\vec{x})=0\}\cap M(r^*)^c$ because $\bar{r}^*$ cannot be positive if $\vec{x}$ manipulates. This suggests the partition of $\mathcal{X}$ into $\mathcal{X}=A\cup (M(\bar{r}^*)\setminus A) \cup (M({r}^*)\setminus A) \cup B$ where $B$ consists features at which agents do not manipulate under both $r^*$ and $\bar{r}^*$.
Then, we have
\begin{multline*}
    L(\bar{r}^*)-L(r^*) \overset{(a)}{\geq} \int_{A,y}(l(1,y) - l(1,y))p(\vec{x},y)\;dyd\vec{x}\\
    + \int_{M({r}^*)\setminus A,y}(l(1,y)-l(1,y))p(\vec{x},y)\;dyd\vec{x} + \int_{M(\bar{r}^*)\setminus A,y} (l(1,y)-c)p(\vec{x},y) \;dyd\vec{x}\\ +  \int_{B\cap \{r^*(\vec{x})\neq \bar{r}^*(\vec{x})\},y}(g_{\bar{r}^*}(\vec{x})-g_{r^*}(\vec{x}))\;dyd\vec{x}
    + \int_{B\cap \{r^*(\vec{x})= \bar{r}^*(\vec{x})\},y}(g_{\bar{r}^*}(\vec{x})-g_{r^*}(\vec{x}))\;dyd\vec{x} \\
    = \int_{M(\bar{r}^*)\setminus A,y} (l(1,y)-c)p(\vec{x},y) \;dyd\vec{x} + \int_{B\cap \{r^*(\vec{x})\neq \bar{r}^*(\vec{x})\},y}(g_{\bar{r}^*}(\vec{x})-g_{r^*}(\vec{x}))\;dyd\vec{x}
\end{multline*}

Here, in (a), the 1st term is because the agents in $A$ manipulate under both $r^*$ and $\bar{r}^*$; the 2nd term is because the agents still get positive outcomes for both cases, by manipulating under $r^*$ while truthfully reporting under $\bar{r}^*$.
Since $r^*$ is the optimal abstention function for the unconstrained case, both remaining terms are non-negative, implying $r^*$ weakly dominates $\bar{r}^*$ for the principal's objective, which is a contradiction.

\vspace{-3pt}
\section{Supplementaries of Section \ref{sec: uniform_example}}

\vspace{-3pt}
\subsection{Proof of Proposition \ref{prop: agent_best_response_case_study}: Agent best response in the uniform case study.}

\begin{proof}
The agent's utility is defined by $U(\hat{x}) = \mathbf{1}_{\hat{x} \geq 0} \cdot \mathbf{1}_{|\hat{x}| \geq T} - \gamma(\hat{x}-x)^2$ . If the agent does not manipulate, their utility is $U_{\text{no manip.}}(\hat{x}) = \mathbf{1}_{\hat{x} \geq 0} \cdot \mathbf{1}_{|\hat{x}| \geq T}$. If the agent manipulates, they will reach the target value $\hat{x} = T$, and get classified as $1$, which yields utility $U(\hat{x}) = \mathbf{1}_{\hat{x} \geq 0} \cdot \mathbf{1}_{|\hat{x}| \geq T} - \gamma(\hat{x}-x)^2 = 1 - \gamma(T-x)^2$. The agent will report $\hat{x} = T$ if $1-\gamma(T-x)^2 > U_{\text{no manip.}}(x)$, which simplifies to when $|T-x| < \sqrt{\frac{1-U_{\text{no manip.}}(x) }{\gamma}}$.

The first case to consider is when $x \in [T, 2]$. In this case the agent's best response is to stay as they are and not manipulate. The second case is when $x \in [-2, T]$. In this region, $U_{\text{no manip.}}(x) <0$ either due to $x<0$ or $|x| < T$. The manipulation condition is $1-\gamma(T-x)^2 > 0$. Let $K = \sqrt{\frac{1}{\gamma}}$. Then we can rewrite the manipulation condition as $|T-x| < K$. This inequality holds for $x \in (T-K, T+K)$. Combining that with the region $x \in [-2, T]$, we get that the agent will manipulate if $x \in (\max(-2, T-K), T)$, and will not manipulate if $x \in [-2, \max(-2, T-K)] $ and $x <T$. Therefore we have derived the full best response for the agent.
\end{proof}

\vspace{-3pt}
\subsection{Proof of Proposition \ref{prop:dm_br}: Principal's optimal threshold $\bar{T}^*$ and corresponding loss in the uniform case study.}

\begin{proof}

We write the decision maker's loss as $L(T) = \mathbb{E}_x[l(x, \hat{x}, T)] = \int_{-2}^{2} l(x, \hat{x}, T) \cdot f(x) dx = \frac{1}{4} \int_{-2}^{2} l(x, \hat{x}, T) dx$ where $l(x, \hat{x}, T) = c \cdot \mathbf{1}_{\{|\hat{x}|<T\}} + ( \mathbf{1}_{\{|x|>T\}} \cdot [ \mathbf{1}_{\{x<0\}} \cdot \mathbf{1}_{\{\hat{x}>0\}} + \mathbf{1}_{\{x>0\}} \cdot \mathbf{1}_{\{\hat{x}<0\}} ] )$.

For any fixed value of $T$, the entire domain of $x$ (from -2 to 2) is partitioned into a finite number of intervals by the points $T$, $\max(-2, T-K)$, 0, and $-T$. Within each of these sub-intervals, $l(x, \hat{x}, T)$ is piecewise constant.

\[ l(x, T) = \begin{cases}
0 & \text{if } x \in [-2, -T] \text{ or } x \in [0, 2] \\
c & \text{if } x \in (-T, \max(-2, T - K)] \\
1 & \text{if } x \in (\max(-2, T - K), 0) \\
\end{cases} \]

The decision maker's objective is to minimize expected loss $L(T)$, which can be written as:
\begin{align*}
L(x, T) &= \int_{-T}^{\max(-2, T-K)} c \,dx + \int_{\max(-2, T-K)}^{0} 1 \,dx
\end{align*}

The critical part of this integral is the $\max(-2, T-K)$. For $T \in [0,2]$, $T-K$ ranges from $-K$ to $2-K$. Since $K > 0$, $2-K < 2$. The form of $L(T)$ changes depending on the relationship between $T$ and $K$. The critical values of $T$ that define these changes are $T = 0$, $T= \frac{K}{2}$ (when $-T = T-K$), $T = K$ (where $T-K = 0$), and $T=2$. We will analyze $L(T)$ and its minimum $\bar{T}^*$ based on the different orders of these critical points, which depend on $K$.

\vspace{-7pt}
\subsubsection{When $0 < K<2$:}
In this case the critical points are ordered $0< \frac{K}{2} < K < 2$. For $0 \leq T \leq \frac{K}{2}$, $T-K \leq -T$, so $\max(-2, T-K) = T-K$. Thus, note that the interval $(-T, T-K)$ is empty, so we are left with $L(T) = \frac{1}{4} \int_{T-K}^0 1 \, dx = \frac{1}{4} (K-T)$. Thus, we see the slope of this region is $m_1 = -\frac{1}{4}$, which indicates $L(T)$ is decreasing. For $\frac{K}{2} < T < K$, we have $\max(-2, T-K) = T-K$, but now $T-K > -T$, so in $L(T)$, both the $c$ term and $1$ term contribute. $L(T) = \int_{-T}^{T-K} c \,dx + \int_{T-K}^{0} 1 \,dx = \frac{1}{4}[T(2c-1)+K(1-c)]$. Here the slope is $m_2 = \frac{2c-1}{4}$. For $K\leq T \leq 2$, the interval $\max(-2, T-K), 0)$ becomes empty, so we are left with $L(T) = \frac{1}{4} \int_{-T}^{T-K} c \, dx = \frac{1}{4} c(2T-K)$. Here the slope is $m_3 = \frac{c}{2}$.

By considering the ranges of $c$, we can identify the decision maker's best response for $0<K<2$. If $c<0.5$, $m_2 < 0$. This indicates that $L(T)$ is first decreasing from $m_1$, then decreasing for $m_2$, then increasing for $m_3$. Thus, the minimum is at the transition from decreasing to increasing, $\bar{T}^* = K$. If c = 0.5, then $L(T)$ first decreases ($m_1$), then becomes constant ($m_2 = 0)$, then increases ($m_3$). The minimum occurs over the flat region $\bar{T}^* = T \in [\frac{K}{2}, K]$. If $c>0.5$, then $L(T)$ first decreases ($m_1$), then becomes increasing ($m_2 > 0 )$, then continues to increase ($m_3$). The minimum occurs at $\bar{T}^* = \frac{K}{2}$.

\vspace{-7pt}
\subsubsection{When $2 \leq K \leq 4$:}
Here, $1 < \frac{K}{2} \leq 2 < K$. Although $K$ lies outside the domain $T \in [0, 2]$, $\frac{K}{2}$ remains within it. For $0 \leq T \leq \frac{K}{2}$, the loss is $L(T) = \frac{1}{4}(K - T)$ with slope $m_1 = -\frac{1}{4}$. For $\frac{K}{2} \leq T \leq 2$, the loss becomes $L(T) = \frac{1}{4}[T(2c - 1) + K(1 - c)]$ with slope $m_2 = \frac{2c - 1}{4}$.

If $c < 0.5$, then $m_2 < 0$ and $L(T)$ decreases throughout; the minimum occurs at $\bar{T}^* = 2$.
If $c = 0.5$, then $m_2 = 0$, so $L(T)$ is constant for $T \in [\frac{K}{2}, 2]$; the minimum is any $\bar{T}^* \in [\frac{K}{2}, 2]$.
If $c > 0.5$, then $m_2 > 0$, so $L(T)$ reaches its minimum at the critical point $\bar{T}^* = \frac{K}{2}$.

\vspace{-7pt}
\subsubsection{When $K>4$:}

In this case, for any $T \in [0,2]$, $T-K < T-4 < -2$. So $\max(-2, T-K) = -2$. Also note that the interval $(-T, -2]$ is empty, so the only integral considered in $L(T)$ is $L(T) = \frac{1}{4}\int_{-2}^0 1 \, dx = \frac{1}{2}$. This indicates that $L(T)$ is constant over $T \in [0, 2]$. Thus in this case, $\bar{T}^* = T \in [0,2]$.

\vspace{-5pt}
\subsubsection{Combining all cases:}
Considering all cases of $K$ and $c$, we compile the following expression for the decision maker's best response:

$$
\bar{T}^* = \begin{cases}
K & \text{if } c < 0.5, 0 < K < 2 \\
2 & \text{if } c < 0.5, 2 \leq K \leq 4 \\
\left[\frac{K}{2}, K\right] & \text{if } c = 0.5, 0 < K < 2 \\
\left[\frac{K}{2}, 2\right] & \text{if } c = 0.5, 2 \leq K \leq 4 \\
\frac{K}{2} & \text{if } c > 0.5, 0 < K \leq 4 \\
[0, 2] & \text{if } K > 4
\end{cases}
$$

Then, it is clear that plugging $\bar{T}^*$ into $L(T)$ yields Equation \ref{eq:dm_br_L}.
\end{proof}

\subsection{Proof of Proposition \ref{prop:no_abs_loss}: Loss with no abstention for the uniform case study}
By considering the agent's best response, we see that:

$$
l(x, \hat{x}(x)) = \begin{cases}
1 & \text{if } -1/\sqrt{\gamma} \leq x < 0 \\
0 & \text{otherwise}
\end{cases}
$$

The expected loss $L$ is found by integrating $l(x, \hat{x}(x)) \cdot f(x)$ over $x$'s domain.
$$L = \int_{-2}^{2} l(x, \hat{x}(x)) \cdot \frac{1}{4} dx$$
Since $l(x, \hat{x}(x)) = 1$ only when $x \in [-K, 0)$, this integral effectively becomes:
$$L = \int_{\max(-2, -K)}^{0} 1 \cdot \frac{1}{4} dx$$
This leads to two distinct cases for the expected loss $L$.
If $K \le 2$, the interval $[-K, 0)$ is entirely contained within our sample space of $x$. 

$$L = \frac{1}{4} \int_{-K}^{0} dx = \frac{1}{4} [x]_{-K}^{0} = \frac{1}{4} (0 - (-K)) = \frac{K}{4}$$
If $K > 2$, and $x$ is defined for $x \in [-2, 2]$, the lower bound for the integral is $-2$. 

$$L = \frac{1}{4} \int_{-2}^{0} dx = \frac{1}{4} [x]_{-2}^{0} = \frac{1}{4} (0 - (-2)) = \frac{2}{4} = \frac{1}{2}$$

\subsection{On expected manipulation for unqualified agents}
\subsubsection{Without abstention:} 

We first compute expected manipulation without abstention. The agent’s best response is $D_{\text{no\_abstention}}(x) = -x$ for $x \in [-K, 0)$ and 0 otherwise. Thus,
\[
E_{\text{no\_abstention}} = \frac{1}{4} \int_{\max(-2, -K)}^0 (-x) \, dx.
\]

\paragraph{Case 1: $0 < K \leq 2$} 
Here, $-K \ge -2$, so $\max(-2, -K) = -K$:
\[
E_{\text{no\_abstention}} = \frac{1}{4} \int_{-K}^{0} (-x) \, dx = \frac{1}{4} \left[ -\frac{x^2}{2} \right]_{-K}^{0} = \frac{K^2}{8}.
\]

\paragraph{Case 2: $K > 2$}
Now $-K < -2$, so $\max(-2, -K) = -2$:
\[
E_{\text{no\_abstention}} = \frac{1}{4} \int_{-2}^{0} (-x) \, dx = \frac{1}{2}.
\]

\subsubsection{With abstention:}

We now compute the expected manipulation from unqualified agents under abstention, by integrating $D_{\text{with\_abstention}}(x) = \bar{T}^* - x$ over $x \in [-2, 0)$, when $\max(-2, \bar{T}^* - K) < x < \bar{T}^*$:
\[
E_{\text{with\_abstention}} = \frac{1}{4}\int_{\max(-2, \bar{T}^*-K)}^{0} (\bar{T}^*-x) \, dx.
\]

\paragraph{Case 1: $0 < K \leq 2$, $c < 0.5$}
Here, $\bar{T}^* = K$, so the integration bounds are $[0, 0]$:
\[
E_{\text{with\_abstention}} = \frac{1}{4}\int_0^0 (K - x)\, dx = 0.
\]

\paragraph{Case 2: $0 < K \leq 2$, $c \ge 0.5$}
Now $\bar{T}^* = \frac{K}{2}$, and bounds are $[-\frac{K}{2}, 0]$:
\[
E_{\text{with\_abstention}} = \frac{1}{4} \int_{-\frac{K}{2}}^{0} \left(\frac{K}{2}-x\right) dx = \frac{3K^2}{32}.
\]

\paragraph{Case 3: $2 \leq K \leq 4$, $c < 0.5$}
Here, $\bar{T}^* = 2$, and bounds are $[2-K, 0]$:
\[
E_{\text{with\_abstention}} = \frac{1}{4} \int_{2-K}^{0} (2 - x)\, dx = \frac{K^2}{8} - \frac{1}{2}.
\]

\paragraph{Case 4: $2 \leq K \leq 4$, $c \ge 0.5$}
Here, $\bar{T}^* = \frac{K}{2}$ and bounds again are $[-\frac{K}{2}, 0]$:
\[
E_{\text{with\_abstention}} = \frac{3K^2}{32} \quad \text{(same as Case 2)}.
\]

\paragraph{Case 5: $K > 4$}
Since $\bar{T}^* - K < -2$, the bounds are $[-2, 0]$, and
\[
E_{\text{with\_abstention}} = \frac{1}{4} \int_{-2}^{0} (\bar{T}^* - x)\, dx = \frac{\bar{T}^*}{2} + \frac{1}{2}.
\]

\subsubsection{Comparing $E_{\text{no\_abstention}}$ and $E_{\text{with\_abstention}}$:}

The first region is $0 < K \le 2$ (i.e., $\gamma \ge 0.25$), where $E_{\text{no\_abstention}} = \frac{K^2}{8}$. In the abstention case, if $c < 0.5$, then $E_{\text{with\_abstention}} = 0$, since the optimal threshold $\bar{T}^* = K$ is too high for unqualified agents to manipulate to. If $c \ge 0.5$, then $E_{\text{with\_abstention}} = \frac{3K^2}{32}$, which is strictly less than $\frac{K^2}{8}$ for $K > 0$. Thus, abstention always reduces expected manipulation in this region.

In the second region, $2 < K \le 4$ ($0.0625 \le \gamma < 0.25$), we have $E_{\text{no\_abstention}} = \frac{1}{2}$. For $c < 0.5$, $E_{\text{with\_abstention}} = \frac{K^2}{8} - \frac{1}{2}$, which is lower than $\frac{1}{2}$ when $K \le \sqrt{8}$, higher when $K > \sqrt{8}$, and equal when $K = \sqrt{8}$. For $c \ge 0.5$, $E_{\text{with\_abstention}} = \frac{3K^2}{32}$, which is less than $\frac{1}{2}$ when $K \le \sqrt{16/3}$, greater when $K > \sqrt{16/3}$, and equal when $K = \sqrt{16/3}$.

In the third region, $K > 4$ (i.e., $\gamma < 0.0625$), $E_{\text{no\_abstention}} = \frac{1}{2}$. With abstention, $E_{\text{with\_abstention}} = \frac{\bar{T}^*}{2} + \frac{1}{2}$ for $\bar{T}^* \in [0, 2]$. If $\bar{T}^* = 0$, expected manipulation matches the no-abstention case; if $\bar{T}^* > 0$, it is strictly higher. Thus, in this region, abstention yields equal or greater expected manipulation, depending on the principal's threshold choice. This occurs because agents can manipulate to any point in $[-2, 2]$ when $K > 4$. In practice, such low $\gamma$ is rare, limiting the relevance of this case.